\documentclass[10pt,twocolumn,letterpaper]{article}
\usepackage[utf8]{inputenc}

\usepackage{amssymb}
\usepackage{amsmath}
\usepackage{mathtools}

\usepackage{algorithm}
\usepackage{algpseudocode}

\usepackage{times}
\usepackage{epsfig}
\usepackage{graphicx}
\usepackage{amssymb}

\usepackage{amsthm}

\usepackage{geometry}
\newtheorem{theorem}{Theorem}[section]

\usepackage[pagebackref=true,breaklinks=true,letterpaper=true,colorlinks,bookmarks=false]{hyperref}

\title{Mixing ADAM and SGD: a Combined Optimization Method}

\author{Nicola Landro, Ignazio Gallo, Riccardo La Grassa\\
University of Insubria\\
\{nlandro, ignazio.gallo, rlagrassa\}@uninsubria.it
}

\begin{document}

\date{}
\maketitle

\begin{abstract}
Optimization methods (optimizers) get special attention for the efficient training of neural networks in the field of deep learning. 
In literature there are many papers that compare neural models trained with the use of different optimizers. 
Each paper demonstrates that for a particular problem an optimizer is better than the others but as the problem changes this type of result is no longer valid and we have to start from scratch.
In our paper we propose to use the combination of two very different optimizers but when used simultaneously they can overcome the performances of the single optimizers in very different problems.
We propose a new optimizer called MAS (Mixing ADAM and SGD) that integrates SGD and ADAM simultaneously by weighing the contributions of both through the assignment of constant weights.
Rather than trying to improve SGD or ADAM we exploit both at the same time by taking the best of both.
We have conducted several experiments on images and text document classification, using various CNNs, and we demonstrated by experiments that the proposed MAS optimizer produces better performance than the single SGD or ADAM optimizers.
The source code and all the results of the experiments are available online at the following link \textit{{{https://gitlab.com/nicolalandro/multi\_optimizer}}}
\end{abstract}

\section{Introduction}
Stochastic Gradient Descent~\cite{robbins1951stochastic} (SGD) is the dominant method for solving optimization problems. 
SGD iteratively updates the model parameters by moving them in the direction of the negative gradient calculated on a mini-batch scaled by the step length, typically referred to as the learning rate. 
It is necessary to decay this learning rate as the algorithm proceeds to ensure convergence. Manually adjusting the learning rate decay in SGD is not easy. To address this problem, several methods have been proposed that automatically reduce the learning rate.
The basic intuition behind these approaches is to adaptively tune the learning rate based only on recent gradients; therefore, limiting the reliance on the update to only a few past gradients. 
ADAptive Moment estimation~\cite{kingma2014adam} (ADAM) is one of several methods based on this update mechanism~\cite{zaheer2018adaptive}.
On the other hand, adaptive optimization methods such as ADAM, even though they have been proposed to achieve a rapid training process, are observed to generalize poorly with respect to SGD or even fail to converge due to unstable and extreme learning rates~\cite{luo2019adaptive}.
To try to overcome the problems of both of these types of optimizers and at the same time try to exploit their advantages, we propose an optimizer that combines them in a new optimizer.

\begin{figure}
    \centering
    \includegraphics[width=1.0\columnwidth]{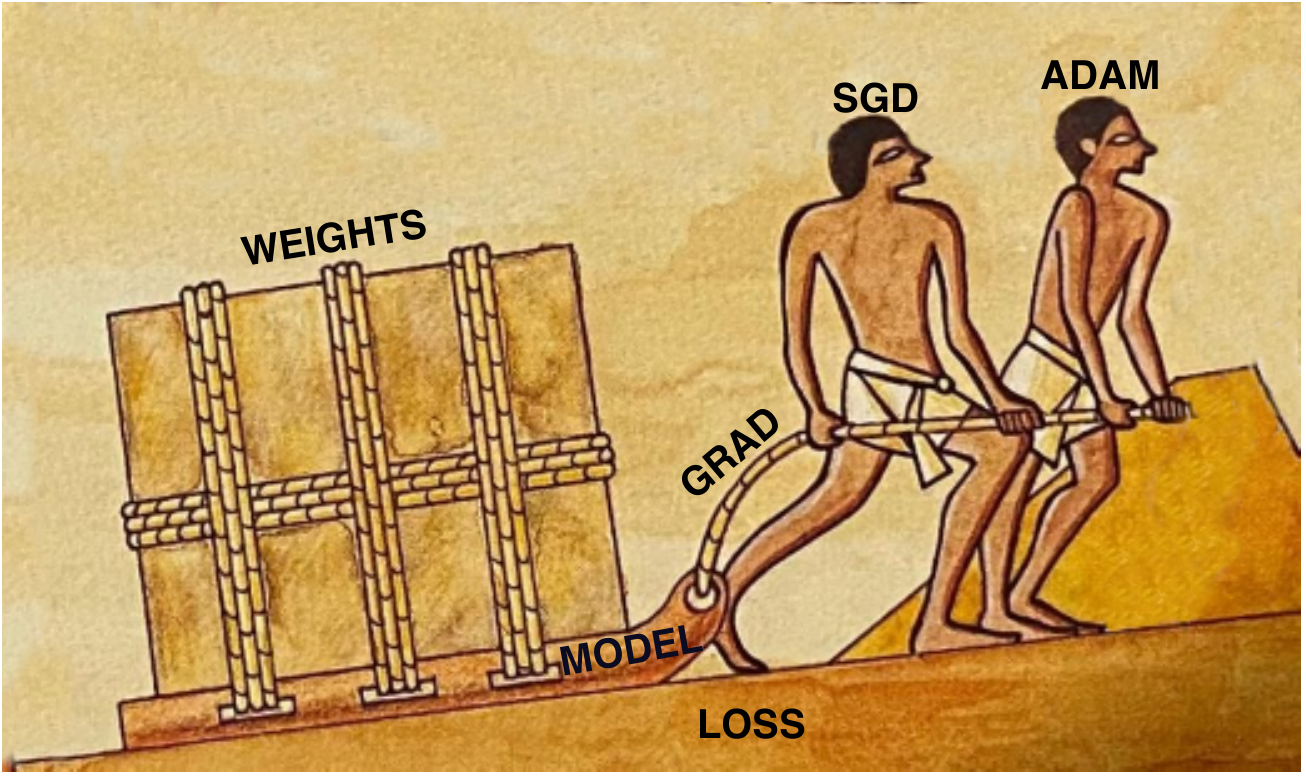}
    \caption{Intuitive representation of the idea behind the proposed MAS (Mixing ADAM and SGD) optimizer: the weights are modified simultaneously by both the optimizers.}
    \label{fig:proposed}
\end{figure}

As depicted in Figure~\ref{fig:proposed}, the basic idea of the MAS optimizer here proposed, is to combine two different known optimizers and automatically go quickly towards the direction of both on the surface of the loss function when the two optimizers agree (see geometric example in Figure~\ref{fig:intuitive}a). While when the two optimizers used in the combination do not agree, our solution always goes towards the predominant direction between the two but slowing down the speed (see example of Figure~\ref{fig:intuitive}b).

Analyzing the literature there are many papers that compare neural models trained with the use of different optimizers~\cite{bera2020analysis,graves2013generating,duchi2011adaptive,zeiler2012adadelta} or that propose modifications for existing optimizers~\cite{luo2018adaptive,kobayashi2020scw,zhang2018improved}, always aimed at improving the results on a subset of problems.
Each paper demonstrates that an optimizer is better than the others but as the problem changes this type of result is no longer valid and we have to start from scratch.
In our paper we propose to combine simultaneously two different optimizers like SGD and ADAM, to overcome the performances of the single optimizers in very different problems.

\textbf{Contributions}.
In light of the above, we set out the main contributions of our paper.
\begin{itemize}
    \item We demonstrate experimentally that the combination of different optimizers in a new optimizer that incorporates them leads to a better generalization capacity in different contexts.
    \item We apply the proposed solution to well-known computer vision and text analysis problems and in both these domains, we obtain excellent results, demonstrating that our solution exceeds the generalization capacity of the starting ADAM and SGD optimizers.
    \item We open the way to this new type of optimizers with which it will be possible to exploit the positive aspects of many existing optimizers, combining them with each others to build new and more efficient optimizers.
    \item To facilitate the understanding of the MAS optimizer and to allow other researchers to be able to run the experiments and extend this idea, we release the source code and setups of the experiments at the following URL~\cite{torch_code}
\end{itemize}

\section{Related Work}
In the literature, there aren't many papers that try to combine different optimizers together.
In this section, we report some of the more recent papers that in some ways use different optimizers in the same learning process.

In~\cite{keskar2017improving} the authors investigate a hybrid strategy, called \textbf{SWATS}, which starts training with an adaptive optimization method and switches to SGD when appropriate.
This idea starts from the observation that despite superior training results, adaptive optimization methods such as ADAM generalize poorly compared to SGD because they tend to work well in the early part of the training, but are overtaken by SGD in the later stages of training.
In concrete terms, SWATS is a simple strategy that goes from Adam to SGD when an activation condition is met.
The experimental results obtained in this paper are not so different from ADAM or SGD used individually, so the authors concluded that using SGD with perfect parameters is the best idea.
In our proposal, we want to combine two well-known optimizers to create a new one that uses simultaneously two different optimizers from the beginning to the end of the training process.

\textbf{ESGD} is a population-based Evolutionary Stochastic Gradient Descent framework for optimizing deep neural networks~\cite{cui2018evolutionary}.  
In this approach, individuals in the population optimized with various SGD-based optimizers using distinct hyper-parameters are considered competing species in a context of coevolution.
The authors experimented with optimizer pools consisting of SGD and ADAM variants where it is often observed that ADAM tends to be aggressive early on, but stabilizes quickly, while SGD starts slowly but can reach a better local minimum. ESGD can automatically choose the appropriate optimizers and their hyper-parameters based on the fitness value during the evolution process so that the merits of SGD and ADAM can be combined to seek a better local optimal solution to the problem of interest.
In the method we propose, we do not need another approach, such as the evolutionary one, to decide which optimizer to use and with which hyper-parameters, but it is the same approach that decides at each step what is the contribution of SGD and that of ADAM.

In this paper, we also compare our MAS optimizer with
\textbf{ADAMW}~\cite{loshchilov2017decoupled}~\cite{loshchilov2018fixing}, which is a version of ADAM in which weight decay is decoupled from $L_2$ regularization.
This optimizer offers good generalization performance, especially for text analysis, and since we also perform some experimental tests on text classification, then we also compare our optimizer with ADAMW.
In fact, ADAMW is often used with BERT~\cite{devlin2018bert} applied to well-known datasets for text classification.


\begin{figure}
    \centering
    \includegraphics[width=0.9\columnwidth]{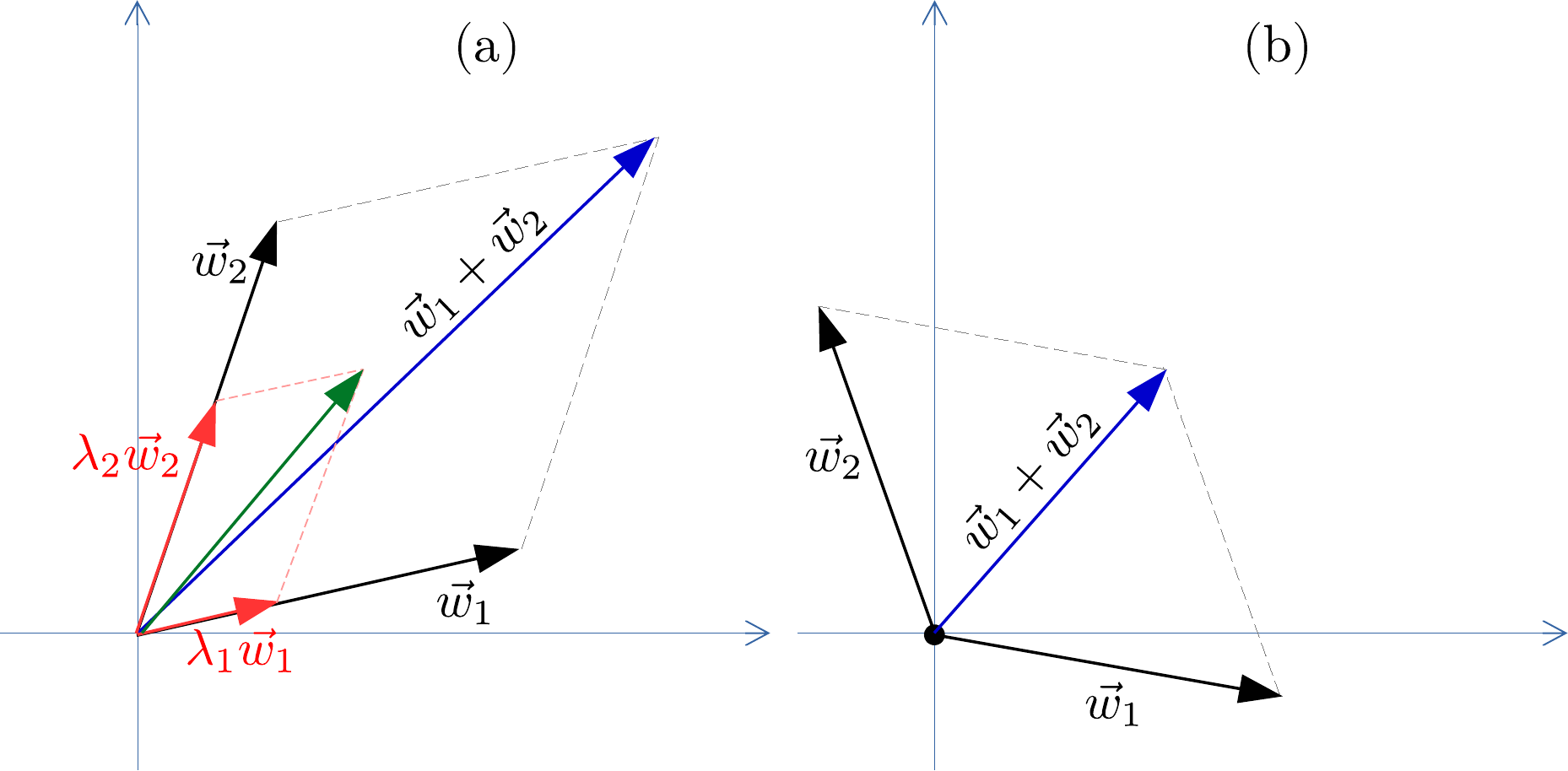}
    \caption{Graphical representation of the basic idea for the proposed MAS optimizer. In (a), if the two translations $\vec{w}_1$ and $\vec{w}_2$ obtained from two different optimizers are similar, then the resulting translation $\vec{w}_1 + \vec{w}_2$ is boosted. In (b), if the translations $\vec{w}_1$ and $\vec{w}_2$ go in two different directions, then the resulting translation is smaller. We also use two hyper-parameters $\lambda_1$ and $\lambda_2$ to weigh the contribution of the two optimizers.}
    \label{fig:intuitive}
\end{figure}

\begin{algorithm} \scriptsize
\caption{Stochastic Gradient Descent (SGD)}
\textbf{Input:} the weights $w_k$, learing rate $\eta$, weight decay $\gamma$, dampening $d$, boolean $nesterov$
\begin{algorithmic}[1]
\State{$v_0=0$}
\Function{$\Delta_\text{SGD}$}{$w_{k}$, $\nabla$, $\gamma$, $\mu$, $d$, $nesterov$}
    \State{$\widehat{\nabla} = \nabla + w_{k} \cdot \gamma$}
    \If{$m \neq 0$}
        \If{$k = 0$}
            \State{$v_k=\widehat{\nabla}$}
        \Else
            \State{$v_k = v_{k-1} \cdot \mu + \widehat{\nabla} \cdot (1 - d)$}
        \EndIf
        \If{$nesterov = True$}
            \State{$v_k = \widehat{\nabla} + v_k \cdot \mu$}
        \EndIf
    \EndIf
    \State{\Return{$v_k$}}
\EndFunction

\For{\texttt{batches}}
    \State{$w_{k+1} = w_{k} - \eta_s \cdot \Delta_\text{SGD}(w_{k}, \nabla, \gamma, \mu, d, nesterov)$}
\EndFor

\end{algorithmic}
\label{alg:sgd}
\end{algorithm}

\section{Preliminaries}
Training neural networks is equivalent to solving the following optimization problem:
\begin{equation}
\min_{w \in \mathbb{R}^n} \mathcal{L}(w)
\end{equation}
where $\mathcal{L}$ is a loss function and $w$ are the weights. 
The iterations of an \textbf{SGD}~\cite{robbins1951stochastic} optimizer can be described as:
\begin{equation}
\label{eq:sgd}
w_{k+1} = w_k - \eta \cdot \nabla\mathcal{L}(w)
\end{equation}
where $w_k$ denotes the weights $w$ at the $k$-th iteration, $\eta$ denote the learning rate and $\nabla\mathcal{L}(w)$ denotes the stochastic gradient calculated at $w_k$.
To propose a stochastic gradient calculated as generic as possible, we introduce the \textbf{weight decay}~\cite{krogh1992simple} strategy, often used in many SGD implementations. 
The weight decay can be seen as a modification of the $\nabla\mathcal{L}(w)$ gradient, and in particular, we can describe it as follows:
\begin{equation}
\label{eq:wd}
    \widehat{\nabla}\mathcal{L}(w_k) = \nabla \mathcal{L}(w_k)+ w_k \cdot \gamma
\end{equation}
where $\gamma$ is a small scalar called weight decay.
We can observe that if the weight decay $\gamma$ is equal to zero then $\widehat{\nabla}\mathcal{L}(w) = \nabla\mathcal{L}(w)$.
Based on the above, we can generalize the Eq.~\ref{eq:sgd}  to the following one that include weight decay:
\begin{equation}
w_{k+1} = w_k - \eta \cdot \widehat{\nabla}\mathcal{L}(w)
\end{equation}
%
%
The SGD algorithm described up to here is usually used in combination with \textbf{momentum}, and in this case, we refer to it as \textbf{SGD(M)}~\cite{sutskever2013importance} (Stochastic Gradient Descend with Momentum).
SGD(M) almost always works better and faster than SGD because the momentum helps accelerate the gradient vectors in the right direction, thus leading to faster convergence.
The iterations of SGD(M) can be described as follows:
\begin{equation}
\label{eq:sgd_vk}
v_{k} = \mu \cdot v_{k-1} + \widehat{\nabla}\mathcal{L}(w)
\end{equation}
\begin{equation}
w_{k+1} = w_k - \eta \cdot v_{k}
\end{equation}
where $\mu \in [0, 1)$ is the momentum parameter and for $k=0$, $v_0$ is initialized to $0$.
%
The simpler methods of momentum have an associated \textbf{damping} coefficient~\cite{damaskinos2018asynchronous}, which controls the rate at which the momentum vector decays.
The dampening coefficient changes the momentum as follow:
\begin{equation}
\label{eq:sgd_vk_dampening}
v_{d_k} = m \cdot v_{k-1} + \widehat{\nabla}\mathcal{L}(w) \cdot (1 - d)
\end{equation}
where $0 \leq d <1$, is the dampening value, so the final SGD with momentum and dampening coefficients can be seen as follow:
\begin{equation}
w_{k+1} = w_k - \eta \cdot v_{d_k}
\end{equation}
%
\textbf{Nesterov} momentum~\cite{liu2018accelerating} is an extension of the moment method that approximates the future position of the parameters that takes into account the movement.
The SGD with nesterov transforms again the $v_k$ of Eq.~\ref{eq:sgd_vk}, more precisely: 
\begin{equation}
\label{eq:sgd_vk_nesterov}
v_{n_k} = \widehat{\nabla}\mathcal{L}(w) + v_{d_k} \cdot m
\end{equation}
\begin{equation}
\label{eq:sgd_vk_nesterov_w_update}
w_{k+1} = w_k - \eta \cdot v_{n_k}
\end{equation}
The complete SGD algorithm, used in this paper, is showed in Alg.~\ref{alg:sgd}.

\begin{algorithm} \scriptsize
\caption{ADaptive Moment Estimation (ADAM)}
\textbf{Input:} the weights $w_k$, learing rate $\eta$, weight decay $\gamma$, $\beta_1$, $\beta_2$, $\epsilon$, boolean $amsgrad$
\begin{algorithmic}[1]
\State{$m_0=0$}
\State{$v^a_0=0$}
\State{$\widehat{v}_0=0$}
\Function{$\Delta_\text{ADAM}$}{$w_{k}$, $\nabla$, $\eta$, $\gamma$, $\beta_1$, $\beta_2$, $\epsilon$, $amsgrad$}
    \State{$\widehat{\nabla} = \nabla + w_{k} \cdot \gamma$}
    \State{$m_k = m_{k-1} \cdot \beta_1 + \widehat{\nabla} \cdot (1-\beta_1)$}
    \State{$v^a_k = v^a_{k-1} \cdot \beta_2 + \widehat{\nabla} \cdot \widehat{\nabla} \cdot (1-\beta_2)$}
    \If{$amsgrad = True$}
        \State{$\widehat{v}_k = \text{max}(\widehat{v}_{k-1}, v^a_k )$}
        \State{$denom = \frac{\sqrt{\widehat{v}_k }}{\sqrt{1 - \beta_2} + \epsilon}$}
    \Else
        \State{$denom = \frac{\sqrt{v_k^a}}{\sqrt{1 - \beta_2} + \epsilon}$}
    \EndIf
    \State{$\eta_a = \frac{\eta}{1 - \beta_1}$}
    \State{$d_k = \frac{m_k}{denom}$}
    \State{\Return{$d_k, \eta_a$}}
\EndFunction

\For{\texttt{batches}}
    \State{$d_k, \eta_a = \Delta_\text{ADAM}(w_{k}, \nabla, \eta, \gamma, \beta_1, \beta_2, \epsilon, amsgrad)$}
    \State{$w_{k+1} = w_{k} - \eta_a \cdot d_k$}
\EndFor

\end{algorithmic}
\label{alg:adam}
\end{algorithm}

\textbf{ADAM}~\cite{kingma2014adam} (ADAptive Moment estimation) optimization algorithm is an extension to SGD that has recently seen broader adoption for deep learning applications in computer vision and natural language processing.
ADAM's equation for updating the weights of a neural network by iterating over the training data can be represented as follows:
\begin{equation}
\label{eq:adam_mk}
m_{k} = 
    \beta_1 \cdot m_{k-1} 
    + (1 - \beta_1) 
        \cdot 
      \widehat{\nabla}\mathcal{L}(w_k)
\end{equation}
\begin{equation}
\label{eq:adam_vk}
v_{k}^a = 
    \beta_2 \cdot v_{k-1} 
    + (1 - \beta_2) 
        \cdot 
      \widehat{\nabla}\mathcal{L}(w_k)^2
\end{equation}
\begin{equation}
w_{k+1} = w_k - \eta \cdot 
    \frac{
        \sqrt{1-\beta_2}
        }{
        1- \beta_1
        } \cdot
    \frac{
            m_{k}
        }{
            \sqrt{v_{k}^a} + \epsilon
        }
\end{equation}
where $m_k$ and $v_k^a$ are estimates of the first moment (the mean) and the second moment (the non-centered variance) of the gradients respectively, hence the name of the method.
$\beta_1$, $\beta_2$ and $\epsilon$ are three new introduced hyper-parameters of the algorithm.
%
\textbf{AMSGrad}~\cite{chen2018closing} is a stochastic optimization method that seeks to fix a convergence issue with Adam based optimizers. 
AMSGrad uses the maximum of past squared gradients $v_{k-1}$ rather than the exponential average to update the parameters:
\begin{equation}
    \widehat{v}_k = \text{max}(\widehat{v}_{k-1}, v_{k}^a)
\end{equation}

\begin{equation}
w_{k+1} = w_k - \eta \cdot 
    \frac{
        \sqrt{1-\beta_2}
        }{
        1- \beta_1
        } \cdot
    \frac{
            m_{k}
        }{
            \sqrt{\widehat{v}_k} + \epsilon
        }
\end{equation}
The complete ADAM algorithm, used in this paper, is showed in Alg.~\ref{alg:adam}.

\section{Proposed Approach}
In this section, we develop the proposed new optimization method called MAS. 
Our goal is to propose a strategy that automatically combines the advantages of an adaptive method like ADAM, with the advantages of SGD, throughout the entire learning process. 
This combination of optimizers is summed as shown in Fig.~\ref{fig:intuitive} where $w_1$ and $w_2$ represent the displacements on the ADAM and SGD on the surface of the loss function, while $w1 + w2$ represents the displacement obtained thanks to our optimizer.
Below we explain each line of the MAS algorithm represented in Alg.~\ref{alg:proposed}.

The MAS optimizer has only two hyper-parameters which are $\lambda_a$ and $\lambda_s$ used to balance the contribution of ADAM and SGD respectively.
In our experiments, we use only one learning rate $\eta$ for both ADAM and SGD, but it is still possible to differentiate between the two learning rates.
In addition to the hyper-parameters typical of the  MAS optimizer, all the hyper-parameters of SGD and ADAM are also needed.
In this paper, we assume to use the most common implementation of gradient descent used in the field of deep learning, namely the mini-batch gradient descent which divides the training dataset into small batches that are used to calculate the model error and update the model coefficients $w_k$.
For each mini-batch, we calculate the contribution derived from the two components ADAM and SGD and then update all the coefficients as described in the three following subsections.

\subsection{ADAM component}
The complete ADAM algorithm is defined in Alg.~\ref{alg:adam}.
In order to use ADAM in our optimizer, we have extracted the $\Delta_{ADAM}$ function which calculates and returns the increments $d_k$ for the coefficients $w_k$, as defined in Eq.~\ref{eq:adam_d}. 
\begin{equation}
    \label{eq:adam_d}
    d_k = \frac{
           \sqrt{1-\beta_2} \cdot m_{k}
        }{
            \sqrt{\widehat{v}_{k}} + \epsilon
        }
\end{equation}
The same $\Delta_{ADAM}$ function also returns the new learning rate $\eta_a$ defined in Eq.~\ref{eq:adam_eta}, useful when a variable learning rate is used.
In this last case, MAS uses $\eta_a$ to calculate a new learning rate at each step.
%
\begin{equation}
    \label{eq:adam_eta}
     \eta_a=\frac{\eta}{1- \beta_1}
\end{equation}
Now, having $\eta_a$  and $d_k$, we can directly modify the weights $w_k$  exactly as done in the ADAM optimizer and described in Eq.~\ref{eq:adam_update_simple}.
\begin{equation}
    \label{eq:adam_update_simple}
    w_{k+1} = w_k - \eta_a \cdot d_{k}
\end{equation}
However, we just skip this last step and use $eta_a$ and $d_k$ for our MAS optimizer.

\subsection{SGD component}
As for the ADAM component, also the SGD component, defined in Alg.~\ref{alg:sgd}, has been divided into two parts: the $\Delta_{SGD}$ function which returns the increment to be given to the weight $w_k$, and the formula to update the weights as defined in Eq.~\ref{eq:sgd_vk_nesterov_w_update}.
The $v_{n_k}$ value returned by the $\Delta_{SGD}$ function is exactly the value defined in Eq.~\ref{eq:sgd_vk_nesterov}, which we will use directly for our MAS optimizer.


\subsection{The MAS optimizer}
The proposed approach can be summarized with the following Eq.~\ref{eq:proposed}
\begin{equation}
    \label{eq:proposed}
w_{k+1} = w_k - 
    ( \lambda_s \cdot \eta + \lambda_a \cdot \eta_a)
    \cdot ( \lambda_s \cdot v_{n_k}
    +
    \lambda_a \cdot d_k)
\end{equation}
where $\lambda_s$ is a scalar for the SGD component and $\lambda_a$ is another scalar for the ADAM component used for balancing the two contribution of the two optimizers. 
$\eta$ is the learning rate of the proposed MAS optimizer, while $\eta_a$ is the learning rate of ADAM defined in Eq.~\ref{eq:adam_eta}.
$d_{k}$ and $v_{n_k}$ are the two increments define in Eq.~\ref{eq:adam_d} and Eq.~\ref{eq:sgd_vk_nesterov} respectively.

Eq.~\ref{eq:proposed} can be expanded in the following Eq.~\ref{eq:proposed_extensive} to make explicit what are the elements involved in the weights update step used by our MAS optimizer.
%
\begin{equation}
    \label{eq:proposed_extensive}
\begin{split}
w_{k+1} = w_k - 
    ( \lambda_s \cdot \eta + \lambda_a \cdot \frac{\eta}{1- \beta_1})
    \cdot \\
    \cdot ( \lambda_s \cdot v_{n_{k}}
    +
    \lambda_a \cdot
    \frac{
           \sqrt{1-\beta_2} \cdot m_{k}
        }{
            \sqrt{\widehat{v}_{k}} + \epsilon
        })
\end{split}
\end{equation}
where $\beta_1$ and $\beta_2$ are two parameters of the ADAM optimizer, $v_{k}^a$ is defined in Eq.~\ref{eq:adam_vk}, and $m_{k}$ is defined in Eq.~\ref{eq:adam_mk}.

\begin{algorithm} \scriptsize
\caption{ Mixing ADAM and SGD (MAS)}
\textbf{Input:} the weights $w_k$, $\lambda_a$, $\lambda_s$, learing rate $\eta$, weight decay $\gamma$, other SGD and ADAM paramiters \dots
\begin{algorithmic}[1]
    \For{\texttt{batches}}
        \State {$d_{k}, \eta_a = \Delta_\text{ADAM}(w_{k}, \nabla, \eta, \gamma, \dots)$}
        \State {$v_{n_k} = \Delta_\text{SGD}(w_{k}, \nabla, \gamma, \dots)$}
        \State {$merged = \lambda_s \cdot v_{n_k} + \lambda_a \cdot d_{k}$}
         \State {$\eta_{m} = \lambda_s \cdot \eta + \lambda_a \cdot \eta_a$}
        \State{$w_{k+1} = w_{k} - \eta_{m} \cdot merged$}
    \EndFor
\end{algorithmic}
\label{alg:proposed}
\end{algorithm}

The MAS algorithm can be easily implemented by following the pseudo code defined in Alg.~\ref{alg:proposed} and by calling the two functions $\Delta_{ADAM}$ defined in Alg.~\ref{alg:adam} and $\Delta_{SGD}$ defined in Alg.~\ref{alg:sgd}.
We can also show that convergence is guaranteed for the MAS optimizer if we assume that convergence has been guaranteed for the two optimizers SGD and ADAM.

\begin{theorem}[MAS Cauchy necessary convergence condition]
\label{th:MAS_cauchy}
If ADAM and SGD are two optimizers whose convergence is guaranteed then the Cauchy necessary convergence condition is true also for MAS.
\end{theorem}

\begin{proof}
Under the conditions in which the convergence of ADAM and SGD is guaranteed~\cite{reddi2019convergence,lee2016gradient}, we can say that $\sum_{k=0}^{p}{\eta \cdot v_{n_k}}$ and $\sum_{k=0}^{p}{\eta_a \cdot d_k}$ converge at $\infty$ . 
That imply the following:
\begin{equation}
    \lim\limits_{p \to \infty} \eta \cdot v_{n_p} 
    = \lim\limits_{p \to \infty} \eta_a \cdot d_p
    = 0 
\end{equation}
We can observe that $\lim\limits_{p \to \infty} \sum_{k=0}^p\eta = \lim\limits_{p \to \infty} \sum_{k=0}^p\eta_a = \infty$ so we can obtain the following:
\begin{equation}
   \label{eq:dp_limit}
   \lim\limits_{p \to \infty} v_{n_p} 
    = \lim\limits_{p \to \infty} d_p
    = 0 
\end{equation}
The thesis is that 
$\sum_{k=0}^{p} (\lambda_s \cdot \eta + \lambda_a \cdot \eta_a) \cdot (\lambda_s \cdot v_{n_k} + \lambda_a \cdot d_k)$ respect the Cauchy necessary convergence condition, so:
$\lim\limits_{p \to \infty} (\lambda_s \cdot \eta + \lambda_a \cdot \eta_a) \cdot (\lambda_s \cdot v_{n_p} + \lambda_a \cdot d_p) = 0$ 
and for Eq.~\ref{eq:dp_limit}, this last equality is trivially true:
\begin{align}
    \label{eq:theorem}
    \lim\limits_{p \to \infty} (\lambda_s \cdot \eta + \lambda_a \cdot \eta_a) \cdot (\lambda_s \cdot v_{n_p} + \lambda_a \cdot d_p) = \nonumber \\
    (\lambda_s \cdot \eta + \lambda_a \cdot \eta_a) \cdot \lim\limits_{p \to \infty}  (\lambda_s \cdot 0 + \lambda_a \cdot 0)
    = 0
\end{align}

\end{proof}
\begin{theorem}
(MAS convercence)
If for $p \to \infty$ is valid that $\sum_{k=0}^{p}{\eta \cdot v_{n_k}}=\eta \cdot m_1$ and $\sum_{k=0}^{p}{\eta_a \cdot d_k}= \eta_a \cdot m_2$ where $m_1 \in \mathbb{R}$ and $m_2 \in \mathbb{R}$ are two finite real values,
then $MAS = \sum_{k=0}^{p} (\lambda_s \cdot \eta + \lambda_a \cdot \eta_a) \cdot (\lambda_s \cdot v_{n_k} + \lambda_a \cdot d_k) = \lambda_s^2 \cdot \eta \cdot m_1 + \lambda_a^2 \cdot \eta_a \cdot m_2 + \lambda_s \cdot \lambda_a \cdot \eta_a \cdot m_1 + \lambda_s \cdot \lambda_a \cdot \eta \cdot m_2$
\end{theorem}
\begin{proof}
We can write MAS series as:
\begin{equation}
\begin{split}
    MAS = \lambda_s^2 \cdot \eta \cdot \sum_{k=0}^p v_{n_k} 
    + \lambda_a^2 \cdot \eta_a \cdot \sum_{k=0}^p d_k
    +\\+ \lambda_s \cdot \lambda_a \cdot \eta_a \cdot \sum_{k=0}^p v_{n_k}
    + \lambda_s \cdot \lambda_a \cdot \eta \cdot \sum_{k=0}^p d_{k}
\end{split}
\end{equation}
This can be rewritten for $p \to \infty$ as:
\begin{equation}
    \label{eq:converg_numb}
    MAS = \lambda_s^2 \cdot \eta \cdot m_1 + \lambda_a^2 \cdot \eta_a \cdot m_2 + \lambda_s \cdot \lambda_a \cdot \eta_a \cdot m_1 + \lambda_s \cdot \lambda_a \cdot \eta \cdot m_2
\end{equation}

\end{proof}

\begin{figure}
    \centering
    \includegraphics[width=.8\columnwidth]{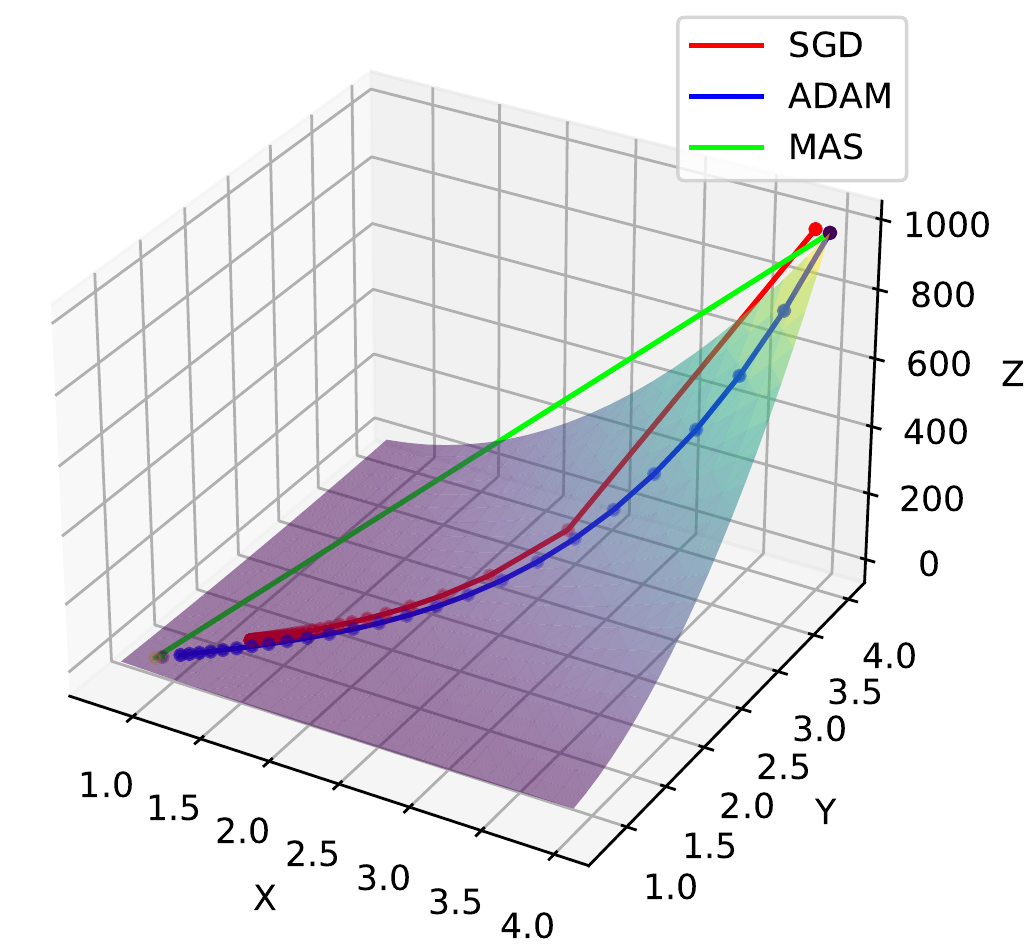}
    \caption{
    Behavior of the three optimizers MAS, ADAM and SGD on the surface defined in Eq.\ref{eq:toy_surface}.
    For better visualization the SGD was shifted on X axis of $0.1$.
    }
    \label{fig:toy}
\end{figure}

\subsection{Geometric explanation}
We can see optimizers as two explorers $w_1$ and $w_2$ who want to explore an environment (the surface of a loss function).
If the two explorers agree to go in a similar direction, then they quickly go in that direction ($w_1 + w_2$).
Otherwise, if they disagree and each prefers a different direction than the other, then they proceed more cautiously and slower ($w_1 + w_2$).
As we can see in Fig.~\ref{fig:intuitive}a, if the directions of the displacement of $w_1$ and $w_2$ are similar then the amplitude of the resulting new displacement $w_1 + w_2$ is increased, however,  as shown in Fig.~\ref{fig:intuitive}b, if the directions of the two displacements $w_1$ and $w_2$ are not similar then the amplitude of the new displacement $w_1 + w_2$ has decreased.

In our approach, the sum $w_1+w_2$ is weighted (see red vectors in Fig.~\ref{fig:intuitive}a) so one of the two optimizers SGD or ADAM  can become more relevant than the other in the choice of direction for MAS, hence the direction resultant may tend towards one of the two.
In MAS we set the weight of the two contributions so as to have a sum $\lambda_1 + \lambda_2 = 1$ in order to maintain a learning rate of the same order of magnitude.

Another important component that greatly affects the MAS shift module at each training step is its learning rate defined in Eq.~\ref{eq:proposed} which combines $\eta$ and $\eta_a$.
The shifts are scaled using the learning rate, so there is a situation where MAS gets more thrust than the ADAM and SGD starting shifts.
In particular, we can imagine that the displacement vector of ADAM has a greater magnitude than SGD and the learning rate of SGD is greater than that of ADAM.
In this case, the MAS shift has a greater vector magnitude than SGD and a higher ADAM learning rate which can cause a large increase in the MAS shift towards the search of a minimum.

\begin{figure}
    \centering
    \includegraphics[width=0.8\columnwidth]{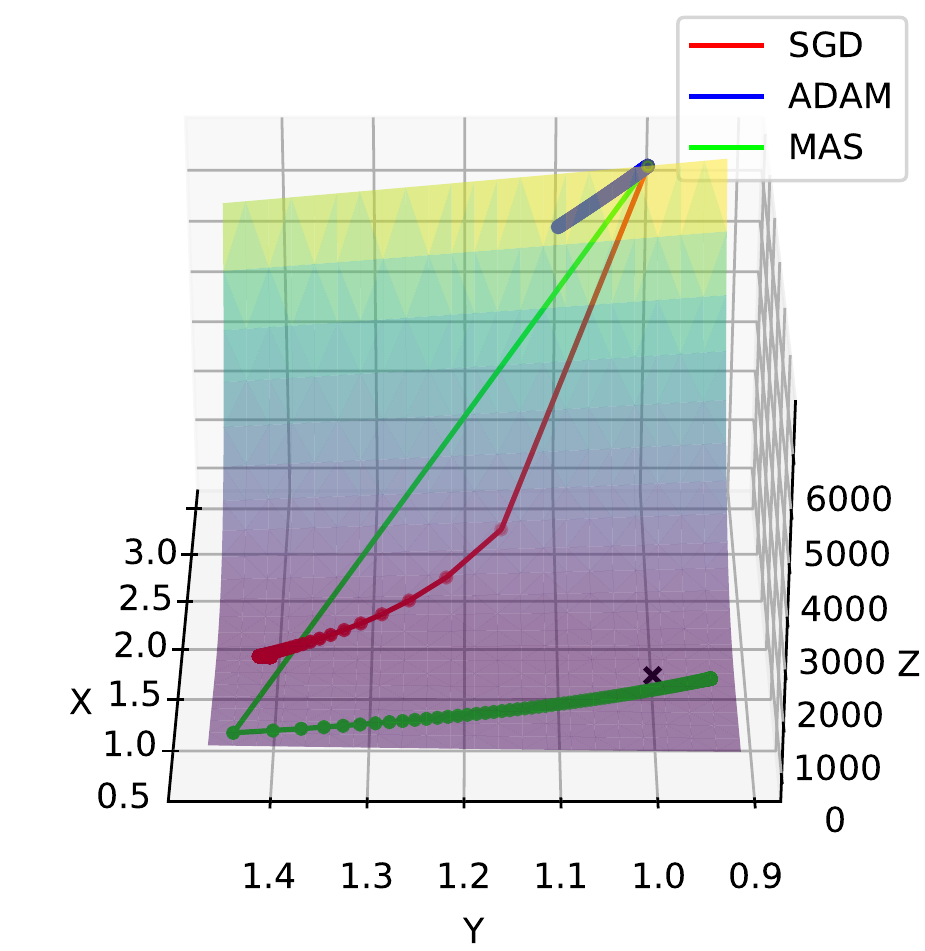}
    \caption{
    Behavior of the three optimizers MAS, ADAM and SGD on the Rosenbrook's surface with $a=1$ and $b=100$}
    \label{fig:toy_rosenbrook}
\end{figure}

\begin{figure}
    \centering
    \includegraphics[width=.8\columnwidth]{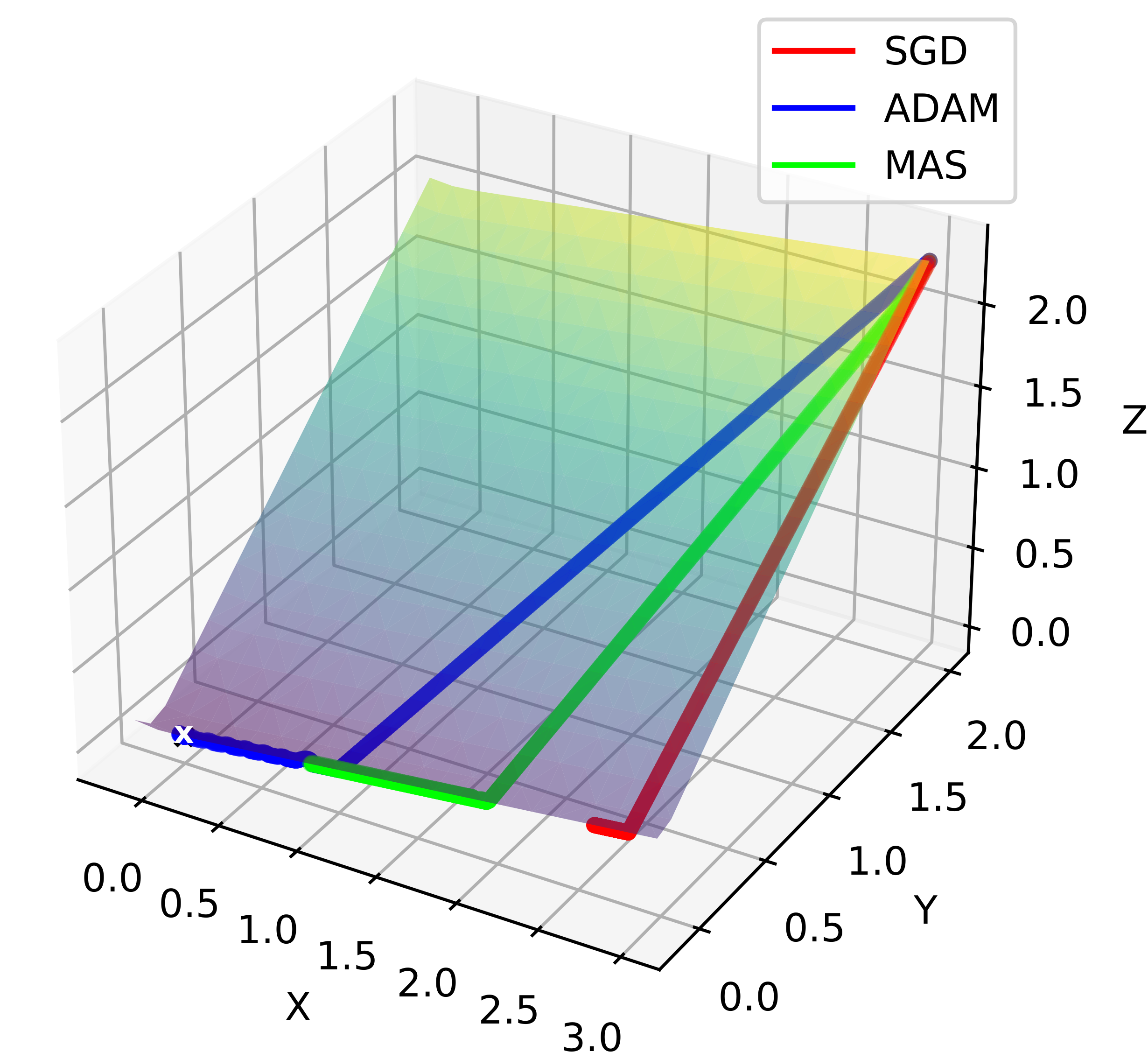}
    \caption{
    Behavior of the three optimizers MAS, ADAM and SGD on the surface
    $z = \frac{\lvert x \lvert} {10} + \lvert y \lvert$ }
    \label{fig:toy_another}
\end{figure}

\subsection{Toy examples}
To better understand our proposal, we built a toy example where we highlight the main behaviour of MAS.
More precisely we consider the following example:
\begin{equation}
    x = [1,2],~y = [2,4]
\end{equation}
\begin{equation}
    p_i = w_1 \cdot(w_2 \cdot x_i)
\end{equation}
\begin{equation}
    \label{eq:toy_surface}
    \mathcal{L}(w_1, w_2) = \sum_{i=0}^1 (p_i - y_i^2)
\end{equation}
We set $\beta_1=0.9$, $\beta_2=0.999$, $\epsilon=10^{-8}$, $amsgrad=False$, dampening $d=0$, $nesterov=False$ and $\mu=0$.
As we can see in Fig.~\ref{fig:toy} our MAS optimizer goes faster towards the minimum value after only two epochs, SGD is fast at the first epoch, however, it decreases its speed soon after and comes close to the minimum after 100 epochs, ADAM instead reaches its minimum after 25 epochs.
Our approach can be fast when it gets a large $v_k$ from SGD and a large $\eta_a$ from ADAM.

Another toy example can be done with the benchmark Rosenbrook~\cite{rosenbrock1960automatic} function:
\begin{equation}
    z = (a - y)^2 + b \cdot (y-x^2)^2
\end{equation}
We set $a=1$ and $b=100$, weight $x=3$ and weight $y=1$, $lr=0.0001$, $epochs=1000$, and default paramiter for ADAM and SGD. 
The MAS optimizer sets $\lambda_s  = \lambda_a = 0.5$.
The comparative result for the minimization of this function is shown in Fig.~\ref{fig:toy_rosenbrook}.
In this experiment, we can see how by combining the two optimizers ADAM and SGD we can obtain a better result than the single optimizers.
For this function, going from the starting point towards the direction of the maximum slope means moving away from the minimum, and therefore it takes a lot of training epochs to approach the minimum.

Let's use a final toy example to highlight the behavior of the MAS optimizer.
In this case we look for the minimum of the function $z = \frac{\lvert x \lvert} {10} + \lvert y \lvert$.
We set the weights $x = 3$  and $y = 2$, $lr =  0.01$, $epochs = 400$  and use all the default parameters for ADAM and SGD.
MAS assigns the same value $0.5$  for the two lamdas hyper-parameters.
In Fig.~\ref{fig:toy_another} we can see how not all the paths between the paths of ADAM and SGD are the best choice. 
Since MAS, as shown in Fig.\ref{fig:intuitive}, goes towards an average direction with respect to that of ADAM and SGD, then in this case ADAM arrives first at the minimum point.

\begin{figure}
    \centering
    \includegraphics[width=\columnwidth]{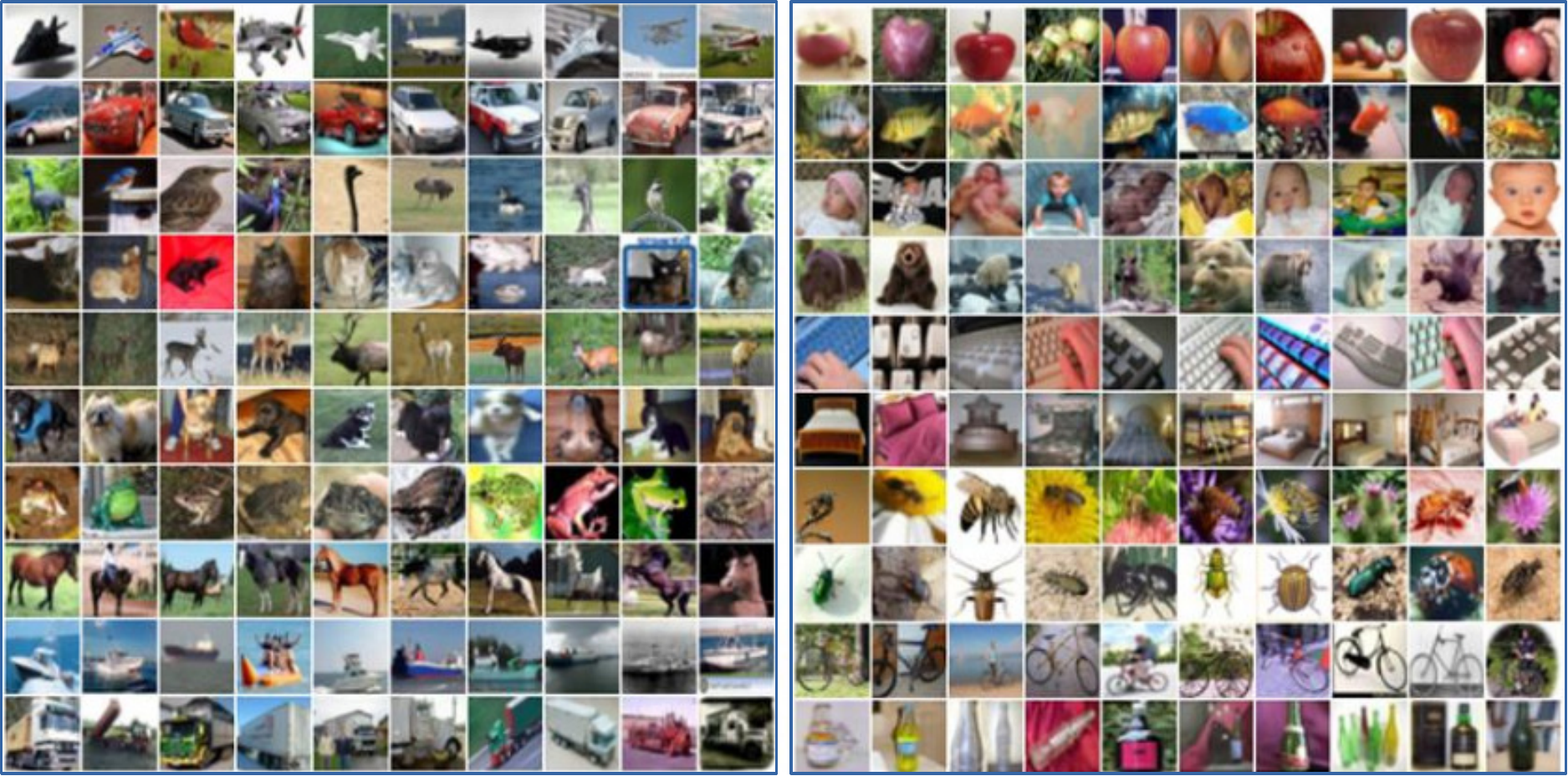}
    \caption{On the left some sample images for each of the 10 classes of CIFAR-10 (one class for each row). On the right 10 classes randomly selected from the set of 100 classes of CIFAR-100.}
    \label{fig:cifar10-100}
\end{figure}

\section{Datasets}
In this section, we briefly describe the datasets used in the experimental phase.

The \textbf{Cifar10}~\cite{cifar} dataset consists of 60,000 images divided into 10 classes (6000 per class) with a training set size and test set size of 50000 and 10000 respectively.
Each input sample is a low-resolution color image of size $32 \times 32$.
The 10 classes are airplanes, cars, birds, cats, deer, dogs, frogs, horses, ships, and trucks.

\begin{table}
    \caption{Accuracy results on Cifar10, after 6 runs of 200 epochs.}
    \label{tab:accuracyCifar10}
    \begin{center}
        \begin{tabular}{llccccc} 
            \hline
            Name & $\lambda_a$ & $\lambda_s$ & avg. acc & acc max\\
            \hline
            Resnet18 &  & & \\
            \hline
            Adam & 1 & 0 & 84.68 & 86.24 \\
            SGD & 0 & 1 & 78.87 & 79.19 \\
            MAS & 0.5 & 0.5 & 85.36 & 85.80 \\
            MAS & 0.4 & 0.6 & 85.64 & 86.59 \\
            MAS & 0.6 & 0.4 & \textbf{85.89} & 86.56 \\
            MAS & 0.7 & 0.3 & 85.39 & 86.09 \\
            MAS & 0.3 & 0.7 & 85.57 & \textbf{86.85} \\
            \hline
            Resnet34 & & & \\
            \hline
            Adam & 1 & 0 & 82.98 & 83.52 \\
            SGD & 0 & 1 & 82.92 & 83.25 \\
            MAS & 0.5 & 0.5 & 84.99 & 85.69 \\
            MAS & 0.4 & 0.6 & \textbf{85.75} & 86.12\\
            MAS & 0.6 & 0.4 & 84.63 &  85.27 \\
            MAS & 0.7 & 0.3 & 84.46 & 84.80 \\
            MAS & 0.3 & 0.7 & 85.71 & \textbf{86.14} \\
            \hline
        \end{tabular}
    \end{center}
\end{table}

The \textbf{Cifar100} \cite{cifar} dataset consist of 60000 images divided in 100 classes (600 per classes) with a training set size and test set size of 50000 and 10000 respectively.
Each input sample is a $32\times 32$ colour images with a low resolution.
In Fig.~\ref{fig:cifar10-100} we report some representative examples of the two datasets, extracted from a subset of classes.

The \textbf{Corpus of Linguistic Acceptability} (CoLA)~\cite{warstadt2018neural} is another dataset which contains 9594 sentences belonging to training and validation sets, and excludes 1063 sentences belonging to a set of tests kept out. In our experiment, we only used the training set and the test set.

The \textbf{AG’s news corpus}~\cite{agnews,zhang2015character} is the last dataset used in our experiments. It is a dataset that contains news articles from the web subdivided into four classes. It has 30,000 training samples and 1900 test samples.

\begin{table}
    \caption{
    Accuracy results on Cifar100, after 7 runs of 200 epochs.}
    \label{tab:accuracyCifar100}
    \begin{center}
        \begin{tabular}{llccccc} 
            \hline
            Name & $\lambda_a$ & $\lambda_s$ & avg acc. & acc max\\
            \hline
            Resnet18 &  & & \\
            \hline
            Adam & 1 & 0 & 49.56 & 50.28 \\
            SGD & 0 & 1 & 49.48 & 50.43 \\
            MAS & 0.5 & 0.5 & 55.08 & 56.68 \\
            MAS & 0.4 & 0.6 & 56.23 & 56.83 \\
            MAS & 0.6 & 0.4 & 53.82 & 54.44 \\
            MAS & 0.7 & 0.3 & 52.92 & 54.07 \\
            MAS & 0.3 & 0.7 & \textbf{58.01} & \textbf{58.48} \\
            \hline
            Resnet34 & & & \\
            \hline
            Adam & 1 & 0 & 50.66 & 51.92 \\
            SGD & 0 & 1 & 52.81 & 53.45 \\
            MAS & 0.5 & 0.5 & 51.52 & 53.26 \\
            MAS & 0.4 & 0.6 & 51.73 & 53.48 \\
            MAS & 0.6 & 0.4 & 52.15 & 53.95 \\
            MAS & 0.7 & 0.3 & \textbf{53.06} & \textbf{54.50} \\
            MAS & 0.3 & 0.7 & 51.96 & 53.32 \\
            \hline
        \end{tabular}
    \end{center}
\end{table}

\section{Experiments}
The optimizer MAS proposed is a generic solution not oriented exclusively to image analysis, so we conduct experiments on both image classification and text document classification.
By doing so, we are able to give a clear indication of the behavior of the proposed optimizer in different contexts, also bearing in mind that many problems, such as audio recognition, can be traced back to image analysis.
In all the experiments $\beta_1=0.9$, $\beta_2=0.999$, $\epsilon=10^{-8}$, $amsgrad=False$, dampening $d=0$ and $nesterov=False$.

\subsection{Experiments with images}
In this first group of experiments we use two well-known image datasets for: (1) conduct an analysis of the two main parameters of MAS, $\lambda_a$ and $\lambda_s$; (2) compare the performance of MAS with respect to the two starting optimizers SGD and ADAM; (3) analyze the behavior of MAS with different neural models.

The datasets used in this first group of experiments are Cifar10 and Cifar100.
The neural models compared are two ResNet and in particular, we use Resnet18 and Resnet34~\cite{He2015,targ2016resnet}.
We analyzed $\lambda_a$ and $\lambda_s$ when they assume values from the set $\{0.3, 0.4, 0.5, 0.6, 0.7\}$ so that $\lambda_a + \lambda_s = 1$.
The numerical results are all grouped in the two Tabs.~\ref{tab:accuracyCifar10} and ~\ref{tab:accuracyCifar100}.

%
For Cifar10 we report the average accuracies calculated on 6 runs of each experiment and for 200 epochs.
We set $\eta = 0.001$, momentum $\mu = 0.95$, and batch size equal to 1024.
From the results reported in Tab.~\ref{tab:accuracyCifar10}, we can observe that the best average results are obtained in correspondence with $\lambda_s = 0.6$ and $\lambda_a = 0.4$ but we can also note that all the other settings used for MAS allow obtaining better results than ADAM and SGD for both Resnet18 and Resnet34.
To better understand what happens during the training phase, in Fig.~\ref{fig:cifar10-test-accuracies} we represent the accuracy of the test and the corresponding loss values of the experiments that produced the best results with Resnet18 on Cifar10.
As we can see even though the best execution of ADAM has the lowest loss value, our optimizer with $\lambda_a = 0.3$ and $\lambda_s = 0.7$ offers better testing accuracy.
The other combination represented with $\lambda_a = 0.5 $ and $\lambda_s = 0.5$ is also very similar to the better accuracy of ADAM and in any case always better than the average accuracy produced by ADAM.
So in general we can say that the MAS optimizer leads to a better generalize than the other optimizers used.

\begin{figure}
    \centering
    \includegraphics[width=0.9\columnwidth]{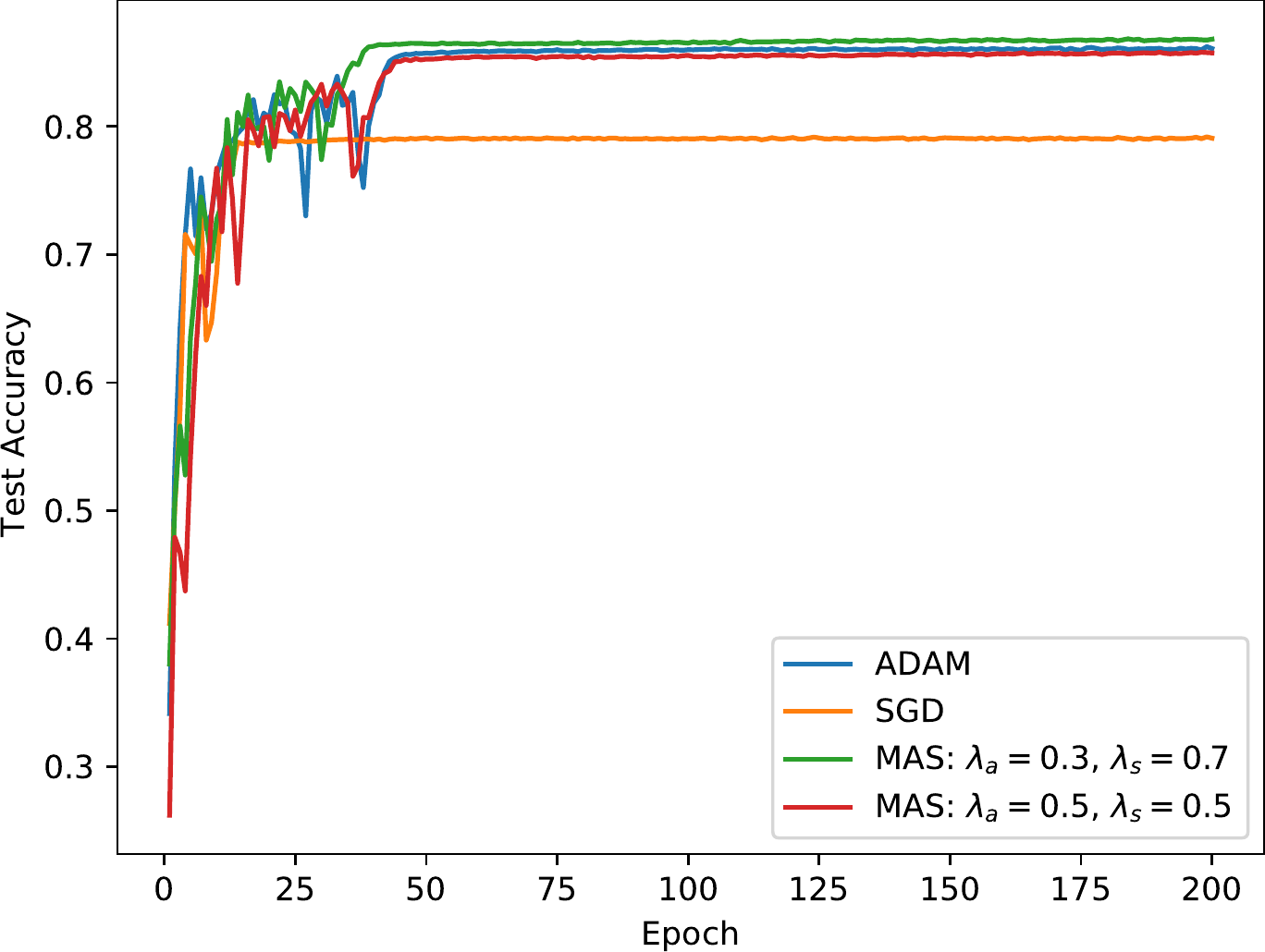}\\
    \includegraphics[width=0.9\columnwidth]{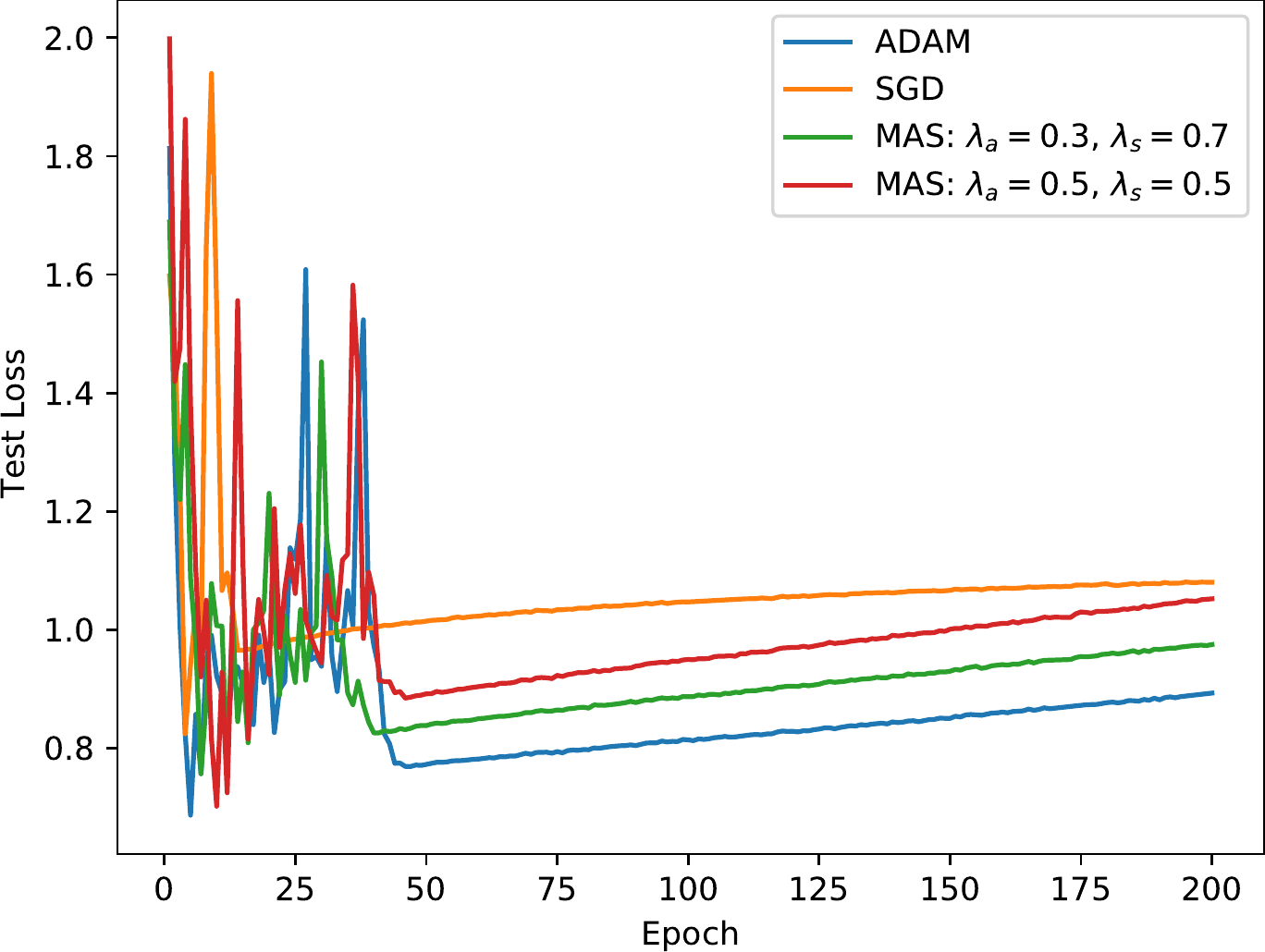}
    \caption{Resnet18 test accuracies and test loss of the best results obtained on Cifar10.}
    \label{fig:cifar10-test-accuracies}
\end{figure}

For Cifar100 we use the same settings as the Cifar10 experiment except for $\eta = 0.008$ and batch size set to 512.
We can see the results in Tab.~\ref{tab:accuracyCifar100}.
Also for this dataset, we can see that for Resnet18 all MAS configurations work better than SGD and ADAM.
Instead, looking at the results obtained with the Resnet34, we can say that only one configuration of MAS exceeds the average accuracies of SGD and ADAM, but if we look at the maximum accuracy values, more than one configuration of MAS is better than the results available with ADAM and SGD.

In conclusion, as we have seen from the results shown in Tab.~\ref{tab:accuracyCifar10} and Tab.~\ref{tab:accuracyCifar100}, the proposed method leads to a better generalization than the other optimizers used in each experiment.
We get better results both by setting $\lambda_a$ and $\lambda_s$ well, and also even when we don't use the best set of parameters.

\subsection{Experiments with text documents}



In this last group of experiments, we use the two datasets of text documents: CoLA and AG's News.
As neural model, we use a model based on BERT~\cite{devlin2018bert} which is one of the best techniques for working with text documents.
To run fewer epochs we use a pre-trained version~\cite{bert_pretrained} of BERT.
In these experiments, we also introduce the comparison with the AdamW optimizer which is usually the optimizer used in BERT-based models.

For the CoLA dataset we set $\eta =  0.0002$, momentum $\mu =  0.95$, and batch size equal to 100. We ran the experiments 5 times for 50 epochs.
For the AG's News dataset we set the same parameters used for CoLA, but we only run it for 10 epochs because it gets good results in the firsts epochs and also because the dataset is very large and therefore takes more time.

We can see all the results in Tab.~\ref{tab:accuracyCoLA}.
Even for text analysis problems, we can confirm the results of the experiments done on images: although AdamW sometimes has better performances than ADAM, our proposed optimizer performs better than other optimizers used in this paper.

\begin{table}
    \caption{Accuracy results of BERT pre-trained on CoLA (50 epochs) and AG's news (10 epochs), after 5 runs.}
    \label{tab:accuracyCoLA}
    \begin{center}
        \begin{tabular}{llccccc} 
            \hline
            Name & $\lambda_a$ & $\lambda_s$ & avg acc. & acc max\\
            \hline
            CoLA &  & & \\
            \hline
            AdamW & - & - & 78.59 & 85.96 \\
            Adam & 1 & 0 & 79.85 & 83.30 \\
            SGD & 0 & 1 & 81.48  & 81.78 \\
            MAS & 0.5 & 0.5 & 85.92 & 86.72 \\
            MAS & 0.4 & 0.6 & 86.18 & \textbf{87.66} \\
            MAS & 0.6 & 0.4 & 85.45 & 86.34  \\
            MAS & 0.7 & 0.3 & 84.66 & 85.78 \\
            MAS & 0.3 & 0.7 & \textbf{86.34} & 86.91  \\
            \hline
            AG's News & & & \\
            \hline
            AdamW & - & - & 92.62 & 92.93 \\
            Adam & 1 & 0 & 92.55 &  92.67 \\
            SGD & 0 & 1 & 91.28 & 91.39\\
            MAS & 0.5 & 0.5 & 93.72 & 93.80 \\
            MAS & 0.4 & 0.6 & 93.82 & 93.98 \\
            MAS & 0.6 & 0.4 & 93.55 & 93.67 \\
            MAS & 0.7 & 0.3 & 93.19 & 93.32 \\
            MAS & 0.3 & 0.7 & \textbf{93.86} & \textbf{93.99} \\
            \hline
        \end{tabular}
    \end{center}
\end{table}

\section{Conclusion}
In this paper, we introduced MAS (Mixing ADAM and SGD) a new Combined Optimization Method that combines the capability of two different optimizers into one. 
We demonstrate by experiments the capability of our proposal to overcome the single optimizers used in our experiments and achieve better performance.
To balance the contribution of the optimizers used within MAS, we introduce two new hyperparameters $\lambda_a$, $\lambda_s$ and show experimentally that in almost all configurations of these parameters, the results are better than the results obtained with the other single optimizers.
In future work, it is possible to change ADAM and SGD and try to mix different optimizers also without the limitations of using only two optimizers.
Another significant future work is to try to change dynamically during the training the influence (lambda hyper-parameters) of the two combined optimizers, in the hypothesis that this can improve further the generalization performance. 


{\small
\bibliographystyle{ieeetr}
\bibliography{main}
}

\end{document}